\documentclass{ecai}
\usepackage{graphicx}
\usepackage{latexsym}

\usepackage[utf8]{inputenc}
\usepackage[T1]{fontenc}
\usepackage[all]{foreign}
\usepackage{amsthm}
\usepackage{amsmath}
\usepackage{amssymb}
\usepackage{thm-restate}

\usepackage{subfigure}
\usepackage{hyperref}
\usepackage[capitalise]{cleveref}
\usepackage{xspace}
\usepackage{thmtools}
\usepackage{mathtools}
\usepackage{autonum}
\usepackage[all]{foreign}
\usepackage{xparse}
\usepackage[table]{xcolor}
\usepackage{forest}
\usepackage{makecell}
\usepackage{multicol}
\usepackage{rotating}
\usepackage{marginnote}
\usepackage{tikz}
\usepackage{algorithm}
\usepackage{algpseudocode}
\usepackage{todonotes}
\usepackage{booktabs}
\usepackage{lipsum}
\usepackage{packages/commands} 
\usepackage{packages/graphics} 
\usepackage{packages/logic}    
\usepackage{packages/symbols}  
\usepackage[inline]{enumitem}
\usepackage{thm-restate}

\usepackage{graphicx}
\usepackage{xspace}
\usepackage{tikz,subfigure}
\usetikzlibrary{arrows,shapes,shapes.multipart, decorations,automata,backgrounds,petri,positioning,shadows,matrix,decorations.pathmorphing, decorations.pathreplacing, decorations.markings, fit,positioning,calc,backgrounds,shapes.misc,arrows.meta,fit}

\usepackage[switch]{lineno}

\usepackage{paralist}
\usepackage{xspace}

\newcommand{\theproperty}{finite memory\xspace}

\newcommand{\m}[1]{\mathsf{#1}}
\newcommand{\mc}[1]{\mathcal{#1}}
\renewcommand{\SS}{\mc S} 
\newcommand{\PP}{\mc P} 
\newcommand{\FF}{\mc F} 
\newcommand{\TT}{\mc T} 
\newcommand{\VV}{\mc V} 
\newcommand{\WW}{\mc W} 
\newcommand\nextvars{\VV^\bigcirc}
\newcommand\wnextvars{\VV^{\bigcirc\kern-.725em\sim}}
\newcommand{\DG}{\textup{DG}} 

\newcommand{\lemref}[1]{Lem.~\ref{lem:#1}}

\newcommand{\defref}[1]{Def.~\ref{def:#1}}

\newcommand{\secref}[1]{Sec.~\ref{sec:#1}}
\newcommand{\tabref}[1]{Tab.~\ref{tab:#1}}

\newcommand{\thmref}[1]{Thm.~\ref{thm:#1}}

\newcommand{\exaref}[1]{Ex.~\ref{exa:#1}}
\newcommand{\figref}[1]{Fig.~\ref{fig:#1}}

\newtheorem{thm}{Theorem}
\newtheorem{example}{Example}
\newtheorem{definition}{Definition}

\newtheorem{corollary}{Corollary}
\newtheorem{proposition}{Proposition}

\crefname{proposition}{Prop.}{Props.}
\Crefname{proposition}{Proposition}{Propositions}

\newcommand{\xmark}{\ding{55}}
\renewcommand{\vec}[1]{\overline{#1}}

\newcommand{\hist}{h}
\newcommand{\combine}[2]{{#1}\,{\circledast}\,{#2}}

\allowdisplaybreaks

\begin{document}

\begin{frontmatter}

\title{Decidable Fragments of LTL\textsubscript{f} Modulo Theories (Extended Version)}

\author[A]{\fnms{Luca}~\snm{Geatti}\orcid{0000-0002-7125-787X}}
\author[B]{\fnms{Alessandro}~\snm{Gianola}\orcid{0000-0003-4216-5199}}
\author[B]{\fnms{Nicola}~\snm{Gigante}\orcid{0000-0002-2254-4821}}
\author[B]{\fnms{Sarah}~\snm{Winkler}\orcid{0000-0001-8114-3107}}

\address[A]{University of Udine, Italy}
\address[B]{Free University of Bozen-Bolzano, Italy}

\begin{abstract}
  We study Linear Temporal Logic Modulo Theories over Finite Traces
  (\LTLfMT), a recently introduced extension of \LTL over finite traces
  (\LTLf) where propositions are replaced by first-order formulas and where
  first-order variables referring to different time points can be compared.
  In general, \LTLfMT was shown to be semi-decidable for any decidable
  first-order theory (\eg linear arithmetics), with a tableau-based
  semi-decision procedure. 

  In this paper we present a sound and complete
  pruning rule for the \LTLfMT tableau.  We show that for any \LTLfMT
  formula that satisfies an abstract, semantic condition, that we call
  \theproperty, the tableau augmented with the new rule is also guaranteed to terminate.
  Last but not least, this technique
  allows us to establish novel decidability results for the
satisfiability of several
  fragments of \LTLfMT, as well as to give new decidability proofs for
  classes that are already known.
\end{abstract}

\end{frontmatter}

\section{Introduction}
\label{sec:intro}

Linear Temporal Logic (\LTL)~\cite{Pnueli77} and its finite-traces counterpart
(\LTLf)~\cite{DeGiacomoV13} are among the most popular formalisms to express
properties of systems both in the formal verification and artificial
intelligence communities. \LTLf has also recently gained traction in business process modeling (BPM)~\cite{DeGiacomoMGMM14,FMW22},
where the real execution of a (business) process is assumed to be always finite.

Due to its propositional nature, \LTL is inherently limited to the
modeling of finite-state systems, while many real-world scenarios, \eg
systems involving numeric data 
or data-aware processes \cite{CalvaneseGGMR19,CGGMR20,CalvaneseGM13}, are better modeled
as infinite-state systems, for which a \emph{first-order} setting is 
needed.
Thus, various first-order extensions of \LTL have been studied
in the literature. 
%
Generally speaking, existing results in this direction are either
purely theoretical (\eg~\cite{KontchakovLWZ04}), or they have been developed
with specific practical scenarios in mind and appear difficult to apply to more
general ones~(\eg~\cite{CimattiGMRT20,DDV12,DLV19}).

As a coherent and principled approach to mitigate this situation, the logic of \emph{\LTLf modulo theories} (\LTLfMT) has
been recently introduced~\cite{GeattiGG22}. \LTLfMT extends \LTLf by
replacing propositions with general first-order formulas interpreted over
arbitrary theories, similar to how \emph{satisfiability modulo theories}
(SMT) extends the Boolean satisfiability problem, and by allowing
comparisons between first-order variables referring to possibly different
time points.

In general, \LTLfMT is undecidable, and it has been shown to be
semi-decidable if applied to decidable first-order fragments and/or
theories~\cite{GeattiGG22,GeattiGG22arXiv}: Crucially, the semi-decidability result has been shown by
providing an effective SMT-based encoding of a tree-shaped tableau that,
once implemented in the BLACK temporal reasoning system \cite{GeattiGM19,geatti2021one},
has proved to work well in practice. Moreover, being theory-agnostic, the
technique works in many different scenarios, leveraging the many expressive
theories, and combinations thereof, supported by modern SMT solvers~\cite{BarrettSST21}. Hence,
\LTLfMT provides a general and theoretically well-founded common ground for
first-order temporal logics that, at the same time, can be applied to
complex scenarios.
The \emph{satisfiability} problem asks whether for a given temporal logic formula $\phi$ there exists a trace that satisfies $\phi$. Satisfiability is a central problem in linear-time temporal logics since a range of key verification tasks, including model checking, can be reduced to it~\cite{RozierV10,LiPZVR20}.

While undecidability is unavoidable when considering expressive infinite-state systems and logics to describe them~\cite{CDMP22,CDMP18,GhilardiGMR22,DeutschLV16},
reasoning and verification has been shown decidable in several
specific cases~\cite{CGGMR20,FMW22,DD07, DDV12}.  It is thus natural to ask which fragments of \LTLfMT have a decidable
satisfiability problem.

In this paper, we address this question in a general way. First, we extend the
tree-shaped tableau for \LTLfMT provided in \cite{GeattiGG22} with a
\emph{pruning rule} that guarantees soundness and \emph{completeness} for any
\emph{decidable} first-order theory, and we give a very general semantic,
sufficient condition, called \emph{\theproperty}, that guarantees that the
tableau, augmented with the new rule, is finite (hence, that its construction
terminates). This equips \LTLfMT with a sound and complete semi-decision
procedure that, in particular, is guaranteed to terminate for any formula that
satisfies the \theproperty property. 

In the next step, we identify a number of syntactic fragments of \LTLfMT that satisfy
the \theproperty property, and are therefore decidable. 
In this way, we both derive novel
decidability results, and recast and generalise existing ones in
this framework. In particular, we prove decidability for \LTLfMT formulas that
either: do not 
compare variables at different time points; only use temporal operators
$\eventually$, $\tomorrow$, and $\weaktomorrow$; belong to a  \emph{bounded
lookback} fragment that restrict variable dependencies in a way to require
only a bounded amount of memory; or that are interpreted over arithmetic
theories but with first-order subformulas restricted to
variable-to-variable/constant comparisons.

A crucial feature of the new pruning rule is that it is sound
and complete in the general case. It is hence always applicable,
avoiding the need to identify the fragment of the input formula beforehand.
This feature will ease implementation (which we leave for future work), because
a single procedure can be implemented, and optimized, that works for a wide
range of decidable fragments as well as for the semi-decidable general case.
%
These results further improve the applicability of \LTLfMT in many scenarios involving complex infinite-state systems, e.g. verification tasks from the areas of knowledge representation  or BPM~\cite{CGGMR20,CalvaneseGGMR19,FMW22,DDV12,DLV19}. Moreover, one may lift the known connection between automated planning and propositional LTL~\cite{BacchusK00,CamachoBMM18} to a first-order, data-aware setting, and use \LTLfMT to address planning problems based on expressive theories. 

\smallskip
The paper is structured as follows. We introduce the relevant background in
\cref{sec:background}. Then, \cref{sec:tableau} provides the new pruning rule
for \LTLfMT and proves that it maintains soundness and completeness.
\cref{sec:decidable} defines the condition of \emph{finite memory}, proves the
termination of the tableau for formulas satisfying such condition, and
identifies a number of decidable fragments of \LTLfMT. Finally,
\cref{sec:conclusions} concludes discussing related work and future directions.

\section{Background}
\label{sec:background}

We consider a given first-order multi-sorted \emph{signature} $\Sigma=\langle\SS,
\PP, \FF, \VV, \WW\rangle$, where $\SS$ is a set of sorts; $\PP$ is a set
of predicate and $\FF$ a set of function symbols; $\VV$ is
a finite, non-empty set of \emph{data variables}; and $\WW$ is a set of
variables disjoint from $\VV$ that will be used for quantification; all
variables are associated with a sort in $\SS$.  Each predicate and function
symbol is supposed to have a type taking sorts from $\SS$; constant symbols
are represented by zero-ary function symbols.  We assume that $\Sigma$
contains equality predicates for all sorts.

Then, $\Sigma$-terms $t$ are built according to the following grammar:
\[
t := v \mid w \mid f(t_1, \dots, t_k) \mid \nextvar v \mid \wnextvar v 
\]
where $v \in \VV$, $w \in \WW$, $f\in \FF$ has arity $k$, and each $t_i$ is a term of appropriate sort. Intuitively, $\nextvar$ and $\wnextvar$ are the \emph{next} and \emph{weak next} operators, that represent the value of a variable $v\in \VV$ in the next state (see the semantics below).
An atom is of the form $p(t_1,\dots, t_k)$, where $p \in \PP$ is a predicate symbol of arity $k$,
and $t_i$ are terms of appropriate sort. Then, \LTLfMT formulas are defined as follows:
\[
\begin{array}{rl}
\lambda &:= a \mid \neg a \mid \lambda_1 \land \lambda_2 \mid \lambda_1 \lor \lambda_2 \mid \exists w.\,\lambda \mid \forall w.\,\lambda \\
\phi &:= \top \mid \lambda \mid \phi_1 \land \phi_2 \mid \phi_1 \lor \phi_2 \mid \tomorrow \phi \mid \weaktomorrow \phi \mid \phi_1 \until \phi_2 \mid \phi_1 \release \phi_2
\end{array}
\]
where $a$ is an atom and $w\,{\in}\, \WW$.
Formulas $\lambda$ as above are called \emph{first-order formulas}.
We call $\phi$ a \emph{state formula} if all its free variables are in $\VV$.
$\Sigma$-formulas without free variables are \emph{$\Sigma$-sentences}, and a
set of $\Sigma$-sentences is a \emph{$\Sigma$-theory} $\TT$.
Note the difference between the \emph{next}
($\nextvar$) and \emph{weak next} ($\wnextvar$) operators, acting on variables, and
the \emph{tomorrow} $(\ltl{X})$, and \emph{weak tomorrow} ($\ltl{wX}$) temporal
operators, acting on formulas.

To define the semantics of first-order formulas, we use the standard notion of a
\emph{$\Sigma$-structure} $M$, which associates each sort $s\in \SS$ with a
domain $s^M$, and each predicate $p\in \PP$ and function symbol $f\in \FF$ with
a suitable interpretation $p^M$ and $f^M$. The equality predicates have the
usual interpretation given by the identity relation. The carrier of $M$,
\ie the union of all domains of sorts in $\SS$, is denoted by $|M|$.
%
%
A function $\alpha\colon \VV \to |M|$ is a \emph{state variable assignment} with respect to $M$, while a function $\gamma\colon \WW\to |M|$ is an \emph{environment},
where we assume in both cases that all variables are mapped to elements of their domain. We write $\gamma[u \mapsto e]$ for the environment $\gamma$ extended with a binding from $u$ to $e$.
A \emph{run} is a pair $\sigma = (M, \langle\alpha_0, \dots,
\alpha_{n-1}\rangle)$ of a $\Sigma$-structure $M$ and a sequence of state
variable assignments with respect to $M$, and $|\sigma|=n$ is its length.

\begin{example}
\label{exa:run}
Let $\VV$ consist of variables $x$ and $y$ of sort \emph{int},
and $M$ be the (unique) model of the theory of linear arithmetic over the integers (\LIA).
Then \eg $(M,\vec \alpha)$ is a run of length 3, for
$\vec \alpha = \langle \{x \mapsto -1, y \mapsto 0\}, \{x \mapsto 0, y \mapsto 1\},\{x \mapsto 2, y \mapsto 2\}\rangle$.
\end{example}

For such a run $\sigma$, some $i$ with $0\leq i < n$, and an environment
$\gamma$, a term $t$ is \emph{well-defined} if $i<n{-}1$, or $t$ does not contain
subterms of the form $\nextvar v$ or $\wnextvar v$. In this case, the evaluation
of the term $t$ is denoted $\eval{t}_{\sigma,\gamma}^i$, and defined as
follows:
\[
\begin{array}{r@{\,}lr@{\,}l}
\eval{v}_{\sigma,\gamma}^i &= \alpha_i(v) & 
\eval{\nextvar v}_{\sigma,\gamma}^i &= \eval{\wnextvar v}_{\sigma,\gamma}^i = \alpha_{i+1}(v) \\
\eval{w}_{\sigma,\gamma}^i &= \gamma(w) &
\eval{f(t_1,\dots,t_k)}_{\sigma,\gamma}^i &= f^M(\eval{t_1}_{\sigma,\gamma}^i, \dots ,\eval{t_k}_{\sigma,\gamma}^i)
\end{array}
\]
where $v\,{\in}\,\VV$ and $w\,{\in}\,\WW$.
Satisfaction of a first-order formula $\lambda$ with respect to an environment $\gamma$ in the run $\sigma$ with $i < |\sigma|$, denoted 
$\sigma \models^i_\gamma \lambda$, is defined as follows:
\begin{alignat}{3}
\sigma \models^i_\gamma & p(t_1, \dots, t_k) && \text{if $t_1, \dots, t_k$ are well-defined and} \\[-1ex]
&&&(\eval{t_1}_{\sigma,\gamma}^i, \dots ,\eval{t_k}_{\sigma,\gamma}^i) \in p^M,
\text{ or}\\[-0.5ex]
&&&\text{if some $t_1, \dots, t_k$ is not well-defined and}\\[-0.5ex]
&&&\text{$t_1, \dots, t_k$ contain $\wnextvar$ but do not contain $\nextvar$}\\
\sigma \models^i_\gamma & \neg p(t_1, \dots, t_k) \quad &&
\text{if }\sigma \not\models^i_\gamma p(t_1, \dots, t_k)\\
\sigma \models^i_\gamma & \lambda_1 \land \lambda_2 &&
\text{if }\sigma \models^i_\gamma \lambda_1 \text{ and }\sigma \models^i_\gamma \lambda_2\\
\sigma \models^i_\gamma & \lambda_1 \lor \lambda_2 &&
\text{if }\sigma \models^i_\gamma \lambda_1 \text{ or }\sigma \models^i_\gamma \lambda_2\\
\sigma \models^i_\gamma & \exists w.\,\lambda &&
\text{if }\sigma \models^i_{\gamma[w\mapsto e]} \lambda
\text{ for some } e\in s^M\\
\sigma \models^i_\gamma & \forall w.\,\lambda &&
\text{if }\sigma \models^i_{\gamma[w\mapsto e]} \lambda
\text{ for all } e\in s^M
\end{alignat}
where $w$ is assumed to have sort $s$.
Satisfaction with respect to $\sigma$ is extended to a general \LTLfMT formula $\phi$ as follows:
\[
\begin{array}{@{}r@{\,}ll@{}}
\sigma \models^i & \lambda &
\text{if } \sigma \models^i_\emptyset \lambda \\
\sigma \models^i & \phi_1 \land \phi_2 &
\text{if }\sigma \models^i \phi_1 \text{ and }\sigma \models^i \phi_2\\
\sigma \models^i & \phi_1 \lor \phi_2 &
\text{if }\sigma \models^i \phi_1 \text{ or }\sigma \models^i \phi_2\\
\sigma \models^i & \tomorrow \phi &
\text{if $i<|\sigma|{-}1$ and }\sigma \models^{i+1} \phi \\
\sigma \models^i & \weaktomorrow \phi &
\text{if $i=|\sigma|{-}1$ or }\sigma \models^{i+1} \phi \\
\sigma \models^i & \phi_1 \until\phi_2 &
\text{if there is some $j$, $i \leq j<|\sigma|$ such that }\sigma \models^{j} \phi_2 \\
&& \text{and }\sigma \models^{k} \phi_1 \text{ for all }i \leq k <j \\
\sigma \models^i & \phi_1 \release\phi_2 &
\text{if either }\sigma \models^{j} \phi_2 \text{ for all }i \leq j<|\sigma|\text{, or there is}\\
&&\text{some $j$, $i \leq j<|\sigma|$ such that }\sigma \models^{j} \phi_1 \\
&& \text{and }\sigma \models^{k} \phi_2 \text{ for all }i \leq k \leq j
\end{array}
\]
Finally, $\sigma$ \emph{satisfies} $\phi$, denoted by $\sigma \models \phi$,  if
$\sigma \models^0 \phi$ holds.
We use the usual shorthands $\eventually \phi \equiv (\top \until \phi)$ and $\always \phi \equiv (\bot \release \phi)$, where $\top \equiv a \vee \neg a$ for any atom $a$ and $\bot\equiv\neg\top$.
For instance, the run in \exaref{run} satisfies
$(y\,{\geq}\,x) \until (x\,{=}\,y)$ and
$\always (\wnextvar x\,{>}x)$, but not $\always (\nextvar x\,{>}x)$ as no first-order formula with $\nextvar$ holds in the last instant.

Let $\nextvars = \{\nextvar v \mid v\in \VV\}$ be the set of all the \emph{next}
variables of $\VV$, and similarly for $\wnextvars$. 
A first-order
formula $\phi$ without $\nextvars \cup \wnextvars$ is satisfied by some
$\Sigma$-structure $M$ and state variable assignment $\alpha\colon V \to |M|$,
denoted $M,\alpha \models \phi$, if $(M,\langle \alpha\rangle) \models \phi$,
which corresponds to the usual notion of first-order satisfaction; if $\phi$ is
a sentence, we simply write $M \models \phi$. 
For a $\Sigma$-structure $M$, we will write $M\in \TT$ to express that $M$ is a
model of $\TT$. A formula is called $\TT$-satisfiable if it is satisfied by some
$\sigma=(M, \vec\alpha)$ with $M\in\TT$.
Moreover, two first-order formulas $\phi_1$ and $\phi_2$ are \emph{$\TT$-equivalent}, denoted
$\phi_1 \equiv_\TT \phi_2$, if $\neg(\phi_1 \leftrightarrow \phi_2)$ is not $\TT$-satisfiable.

A $\Sigma$-theory $\TT$ has \emph{quantifier elimination} (QE) if for any
$\Sigma$-formula $\phi$ there is a quantifier-free formula $\phi'$ that is
$\TT$-equivalent to $\phi$.

%
In the paper we will sometimes refer to common SMT theories~\cite{BarrettSST21}: the theory of equality and uninterpreted functions for a given $\Sigma$ (\EUF), linear arithmetics over rationals (\LRA) and integers (\LIA).


\paragraph{Tableau for \LTLfMT}

We now recall the one-pass tree-shaped tableau for \LTLfMT presented
in~\cite{GeattiGG22}.  The closure of a formula $\phi$, denoted
$\closure{(\phi)}$, is the smallest set of formulas that contains all
subformulas of $\phi$, and, in addition, $\tomorrow (\phi_1 \until
\phi_2)$ whenever $\phi_1 \until \phi_2 \in \closure{(\phi)}$ and
$\weaktomorrow (\psi_1 \release \psi_2)$ whenever $\psi_1 \release \psi_2
\in \closure{(\phi)}$.
A \emph{tableau} for an \LTLfMT formula $\phi$ is a rooted tree in which each
node $u$ is labelled by a set of formulas $\Gamma(u) \subseteq \closure(\phi)$,
as follows. The root node $u_0$ has label $\Gamma(u_0) = \{\phi\}$, and every
other node is the result of applying one of a set of rules to its parent. If any
is applicable, one of the \emph{expansion rules}, shown in
\tabref{expansion:rules}, is applied.

\begin{table}
  \centering
\begin{tabular}{llll}
\toprule
rule & $\phi \in \Gamma(u)$ & $\Gamma_1(\phi)$ & $\Gamma_2(\phi)$ \\
\midrule
\textsf{DISJUNCTION} & $\psi \lor \chi$ & $\{\psi\}$ & $\{\chi\}$ \\
\textsf{CONJUNCTION} & $\psi \land \chi$ & $\{\psi,\ \chi\}$ \\
\textsf{UNTIL} & $\psi \until \chi$ & $\{\chi\}$ & $\{\psi,\ \tomorrow (\psi \until \chi)\}$\\
\textsf{RELEASE} & $\psi \release \chi$ & $\{\psi,\ \chi\}$ & $\{\chi,\ \weaktomorrow (\psi \release \chi)\}$\\
\bottomrule
\end{tabular}
\caption{Expansion rules for \LTLfMT tableau.\label{tab:expansion:rules}}
\end{table}

When applying a rule to a formula $\phi \in \Gamma(u)$ for a node $u$, two children $u_1$ and $u_2$ of $u$ are constructed, which are labeled 
$\Gamma(u) \setminus \{\phi\} \cup \Gamma_1(\phi)$ and 
$\Gamma(u) \setminus \{\phi\} \cup \Gamma_2(\phi)$, respectively, with the second child omitted if $\Gamma_2(\phi)$ is empty.
If no expansion rule is applicable to a node $u$, the node is called
\emph{poised}. By definition of the expansion rules, such a node can contain
only atoms, or temporal formulas rooted by $\tomorrow$ and $\weaktomorrow$.
Poised nodes represent a state in a possible model for the formula. Then,
\emph{time} advances, from a poised node $u$, by applying the $\m{STEP}$ rule,
which creates a child $u'$ of $u$ such that:
\[
\m{STEP}\colon\qquad
\Gamma(u') = \{ \psi \mid 
\tomorrow \psi \in \Gamma(u)\text{ or }
\weaktomorrow \psi \in \Gamma(u)
\}
\]
However, the $\m{STEP}$ rule is only applied if the branch is not ready to be
either \emph{accepted} or \emph{rejected} by one of two \emph{termination
rules}. These rules are defined, for a branch $\vec u$, via a first-order
formula $\Omega(\vec u)$ which summarizes all constraints along the branch. The
formula is defined over the signature $\Sigma' = (\SS, \PP', \FF, \VV^\omega,
\WW)$, where $\PP' = \PP \cup \{\ell \}$ for some fresh $\ell$, and $\VV^\omega
= \bigcup_{i\in \mathbb N} V^i$ where $V^i = \{v^i \mid v\in V \}$ are indexed
versions of the variables in $\VV$. We write $\vec \VV$ for the list of
variables $(v_1, \dots, v_k)$, ordering the variables in $\VV$ in some arbitrary
but fixed way; and similarly, $\vec V^i$ for $(v_1^i, \dots, v_k^i)$.

The stepped version $t^{(i)}$ of an arbitrary term $t$ is defined as follows:
\begin{inparaenum}
\item $w^{(i)} = w$ for all $w\in \WW$;
\item $v^{(i)} = v^i$ for all $v\in \VV$;
\item $(\nextvar x)^{(i)} = (\wnextvar x)^{(i)} = x^{i+1}$; and
\item $f(t_1, \dots, t_n)^{(i)} = f(t_1^{(i)}, \dots, t_n^{(i)})$.
\end{inparaenum}
We extend the notion to formulas, and set $\psi^{(i)}$ to the formula obtained
from $\psi$ by replacing each term $t$ in $\psi$ by $t^{(i)}$. The role of
$\ell$ is to denote the last position of a run. Given a first-order formula
$\phi$, the formula $L(\phi)$ is obtained by replacing all atoms $A$ containing
any term from $\nextvars$ by $\ell \land A$, and all atoms $B$ containing any
term from $\wnextvars$ (but not from $\nextvars$) by $\ell \to B$. 

More generally, we define $\Omega$ for sequences of constraints.
Let $\vec C=\seq{C_0, \dots, C_{m-1}}$ be a sequence of first-order formulas with free variables $\VV\cup \nextvars \cup \wnextvars$.
Then $\Omega(\vec C)$ is defined as
\[
\Omega(\vec C) = \bigwedge_{i=0}^{m-2} C_i^{(i)} \wedge L(C_{m-1})^{(m-1)}
\]
Notice that, according to the definition of $\psi^{(i)}$, 
only variables from $V$ 
are stepped, whereas the $\ell$ atom is left unchanged.
For a branch $\vec u$ with poised nodes $\vec \pi = \langle \pi_0, \dots, \pi_{m-1}\rangle$ and
$F(\pi_i)$ the conjunction of first-order formulas in $\pi_i$,
we set $\Omega(\vec u) = \Omega(\langle F(\pi_0), \dots, F(\pi_m-1)\rangle)$.%
\footnote{We use a slightly modified but equivalent variant of the definition of $\Omega$ from \cite{GeattiGG22}, applying the $L$ operator only to the last instant. This allows us to use a single constant $\ell$, which will simplify the definition of the \textsf{PRUNE} rule.}
Intuitively, $\Omega(\vec u)$ serves the purpose to capture a candidate model along the branch $\vec u$.

Given $\Omega(\vec u)$, the termination rules are defined as follows. The $\m{EMPTY}$ rule is responsible for acceptance:
\[
\begin{array}{@{}ll@{}}
\m{EMPTY}\colon&
\text{If $\Gamma(\pi_{m-1})$ does not contain formulas rooted by $\tomorrow$} \\
&\text{and $\Omega(\vec u) \land \neg \ell$ is satisfiable, the branch is accepted.}
\end{array}
\]

\noindent
Whereas the $\m{CONTRADICTION}$ rule is responsible for rejection:
\[
\displaystyle
\begin{array}{@{}ll@{}}
\m{CONTRADICTION}\colon
&\text{If $\Omega(\vec u)$ is $\TT\cup \EUF$-unsatisfiable, the}\\
&\text{branch $\vec u$ is rejected.}
\end{array}
\]

From \cite{GeattiGG22}, we can state the soundness and completeness of the
tableau for \LTLfMT so defined.
\begin{proposition}[\cite{GeattiGG22}]
  \label{prop:tableau}
  A tableau for an \LTLfMT formula $\phi$ contains an accepted branch if and
  only if the formula is satisfiable.
\end{proposition}

The construction of the tableau for an arbitrary formula is not guaranteed to terminate,
which is to be expected since \LTLfMT is in general undecidable.
However, since accepted
branches are finite, if the formula is satisfiable, a breadth-first construction
of the tree will surely find an accepted branch. Hence, for decidable theories, this tableau
provides a \emph{semi-decision} procedure for \LTLfMT satisfiability.
\begin{proposition}[\cite{GeattiGG22}]
  \label{prop:semi-decidability}
  \LTLfMT satisfiability is semi-decidable.
\end{proposition}

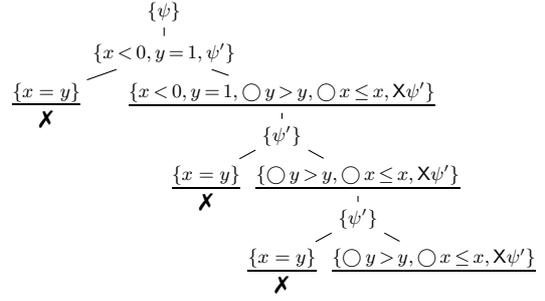
\begin{figure}[t]
  \centering
  \begin{tikzpicture}[level 2/.style={sibling distance=3.1cm}, level 4/.style={sibling distance=2cm}, node distance=3mm, level distance=5.5mm]
    \tikzstyle{tableau}=[scale=.8]
    \tikzstyle{formula}=[blue, scale=.8, anchor=west, yshift=2mm]
    \node[tableau] (0) {$\{\psi\}$}
     child {
      node[tableau] (1x) {$\{x\,{<}\,0, y\,{=}\,1, \psi'\}$}
      child { 
        node[tableau] (c1){\underline{$\{x=y\}$}} node[below=1mm]{\xmark} edge from parent[]}
     child {
      node[tableau] (1) {$\underline{\{x\,{<}\,0, y\,{=}\,1, \nextvar y\,{>}\,y, \nextvar x\,{\leq}\,x, \tomorrow\psi'\}}$}
      child {
      node[tableau] (1a) {$\{\psi'\}$}
      child { 
        node[tableau] (c2){\underline{$\{x=y\}$}} node[below=1mm]{\xmark} edge from parent[]}
      child {
        node[tableau] (2) {$\underline{\{\nextvar y\,{>}\,y, \nextvar x\,{\leq}\,x, \tomorrow\psi'\}}$}
      child {
      node[tableau] (2a) {$\{\psi'\}$}
      child { 
        node[tableau] (c3){\underline{$\{x=y\}$}} node[below=1mm]{\xmark} edge from parent[]}
      child {
        node[tableau] (3) {$\underline{\{\nextvar y\,{>}\,y, \nextvar x\,{\leq}\,x, \tomorrow\psi'\}}$}
      } edge from parent
     } edge from parent
    } edge from parent
    }}}
  ;
  \end{tikzpicture}
  \caption{Example tableau for the \LTLfMT formula $\psi$ of \exaref{1}. Poised nodes are underlined, and nodes marked {\xmark} are rejected.}
  \label{fig:example}
\end{figure}

\begin{example}
\label{exa:1}
Consider the following formula:
\begin{equation}
  \psi \coloneqq (x\,{<}\,0 \wedge y\,{=}\,1) \wedge ((\nextvar y\,{>}\,y
\wedge \nextvar x\,{\leq}\,x) \until x=y)
\end{equation} 
interpreted over \LRA. A partial tableau for $\psi$ is shown in
\figref{example}, where $\psi' \coloneqq (\nextvar y\,{>}\,y \wedge \nextvar
x\,{\leq}\,x) \until x=y$. Note that $\psi$ is unsatisfiable, but the
\textsf{CONTRADICTION} rule is not sufficient to conclude that, as the
right-most branch is going to expand forever.
\end{example}

\section{A new pruning rule for the \LTLfMT tableau}
\label{sec:tableau}

As discussed in \exaref{1}, the right-most branch of \figref{example} is the
prototypical example of a branch that expands forever because of some
unfulfillable request that is postponed forever without ever causing a local
contradiction. In Reynolds' tree-shaped tableau for propositional \LTL, this
case is handled by an ad-hoc $\m{PRUNE}$ rule, which takes care of rejecting
such branches~\cite{Reynolds16a}. 
Here, we define a similar rule for \LTLfMT. 
To this end,
we use a quantified
variant of the formula $\Omega(\vec C)$, for a sequence of first-order formulas
$\vec C$. 

\begin{definition}[History constraints]
  The \emph{history constraint} of a sequence of first-order formulas $\vec C$, denoted $\hist(\vec C)$, is defined as:
  \begin{equation}
    \hist(\vec C) = \begin{cases}
      \top & \text{if $\vec C$ is empty} \\
      (\exists V^0 \dots V^{m-1}.\ \Omega(\vec C))[\vec V^m\!\! \mathbin{/} \vec \VV] & \text{otherwise}
    \end{cases}
  \end{equation}
\end{definition}
That is, all stepped variables are existentially quantified except for the last
ones, which are renamed to $\vec \VV$, so that $\hist(\vec C)$ is a formula with
free variables $\VV$. For a branch $\vec u$ with poised nodes $\vec \pi =
\langle \pi_0, \dots, \pi_{m-1}\rangle$, let $h(\vec \pi) = h(\langle F(\pi_0), \dots, F(\pi_{m-1})\rangle)$. Intuitively, the history
constraint of a branch $\vec u$ summarises all constraints accumulated along the
branch, just like $\Omega$, but by existentially quantifying all variables except those in the last instant, it expresses the
\emph{effect} of the accumulated constraints (the history) on the variables
$\VV$. If the theory under consideration has quantifier elimination (QE), history
constraints are always equivalent to quantifier-free formulas.

%

\begin{example}
\label{exa:2}
Let $\langle\pi_0, \pi_1, \pi_2\rangle$ be the poised nodes in the right-most
branch of the tableau in \figref{example}, and denote as $\vec \pi_{\leq
i}$, for $0 \leq i \leq 2$,  the branches up to these nodes. Then, we have:
\begin{align}
\hist(\vec \pi_{\leq 0}) &= (\exists x_0\,y_0.\ 
 x_0\,{<}\,0 \land y_0\,{=}\,1 \land y\,{>}\,y_0\land x\,{\leq}\,x_0 \land \ell) \\
 & \equiv_\LRA x < 0 \wedge y > 1 \land \ell\\
\hist(\vec \pi_{\leq 1}) &= \exists x_1\,y_1\,x_0\,y_0.\ 
 x_0\,{<}\,0 \land y_0\,{=}\,1 \land y_1\,{>}\,y_0\land x_1\,{\leq}\,x_0 \land {}\\
 & \qquad y\,{>}\,y_1 \wedge x\,{\leq}\,x_1 \land \ell \\
 &\equiv_\LRA x < 0 \wedge y > 1 \land \ell\\
\hist(\vec \pi_{\leq 2}) &\equiv_\LRA \exists x_2\,y_2.\ x_2 < 0 \wedge y_2 > 1
 \land y\,{>}\,y_2 \wedge x\,{\leq}\,x_2 \land \ell \\
 &\equiv_\LRA x < 0 \wedge y > 1 \land \ell
\end{align}
Here the equivalences are obtained with quantifier elimination in \LRA,
so all history constraints are $\LRA$-equivalent.
This reflects the fact that what can be said about $x$ and $y$ after the respective nodes is always the same: $x$ is negative, and $y$ is greater than 1.
\end{example}

Intuitively, if the labels and history constraints of nodes repeat, no progress is made on this branch.
This motivates the 
next
definition.

Given a tableau branch $\vec u$ with poised nodes $\vec \pi = \langle \pi_0, \dots, \pi_{m-1}\rangle$:
\[
\begin{array}{@{}ll@{}}
\m{PRUNE}\colon
 &\text{If  $\Gamma(\pi_i) = \Gamma(\pi_{m-1})$ for some $i\,{<}\,m$ and}\\
 &\text{$\hist(\vec \pi) \models_\TT \hist(\vec \pi_{\leq i})$ then $\vec u $ is rejected.}
\end{array}
\]

Testing whether the $\m{PRUNE}$ rule applies requires 
to check entailment in the underlying theory $\TT$.
If $\TT$ is decidable, this is always possible (\eg if $\TT$ is \LIA or \LRA).
However, in \secref{decidable} we show that even for theories where this is not feasible in general, $\m{PRUNE}$ can be applied in a number of special cases.
Moreover, note that the entailment condition of the $\m{PRUNE}$ rule is
equivalent to saying that 
the set of states described by the formula $\hist(\vec \pi)$
(which represents the history effect at the end of $\pi$) is contained in the
set of states described by the formula $\hist(\vec \pi_{\leq i})$ (representing
the effect 
up to~instant~$i$).

Finally, note that even though there is an apparent overlap between the
definitions of the $\m{EMPTY}$ and $\m{PRUNE}$ rules, the two can never be
applicable together on the same node, because in this case,
$\m{EMPTY}$ would have triggered before (on the repeated node
identified by $\m{PRUNE}$), and the branch would have been already
accepted.

The rightmost branch in \figref{example} is rejected by the \textsf{PRUNE} rule:
for $\pi_1$ and $\pi_2$ the last two poised nodes on the branch, 
$\Gamma(\pi_1) = \Gamma(\pi_2)$ holds and, as shown in \exaref{2}, $\hist(\vec \pi_{\leq 1})$ and $\hist(\vec \pi_{\leq 2})$ are \LRA-equivalent.
A further example of an application of the rule follows.

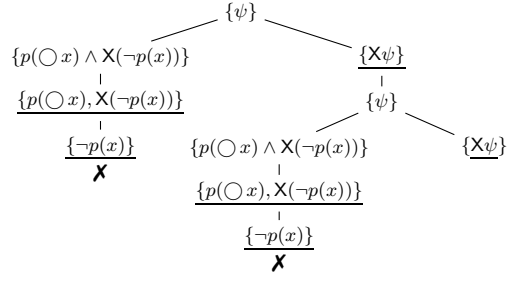
\begin{figure}[t]
  \centering
  \begin{tikzpicture}[level 1/.style={sibling distance=37mm}, level 3/.style={sibling distance=27mm}, level 5/.style={sibling distance=17mm}, node distance=3mm, level distance=6mm]
    \tikzstyle{tableau}=[scale=.8]
    \tikzstyle{formula}=[blue, scale=.8, anchor=west, yshift=2mm]
    \node[tableau] (0) {$\{\psi\}$}
      child { 
        node[tableau] (c1){$\{p(\nextvar x) \wedge \tomorrow(\neg p(x))\}$} 
        child { 
          node[tableau] (c11){\underline{$\{p(\nextvar x),\tomorrow(\neg p(x))\}$} }
            child { 
            node[tableau] (c111){\underline{$\{\neg p(x)\}$} }
            node[below=1mm]{\xmark} 
            }
          }
        }
     child {
      node[tableau] (1) {$\underline{\{ \tomorrow\psi\}}$}
      child {
      node[tableau] (1a) {$\{\psi\}$}
      child { 
        node[tableau] (c2){$\{p(\nextvar x) \wedge \tomorrow(\neg p(x))\}$} 
        child { 
          node[tableau] (c21){\underline{$\{p(\nextvar x),\tomorrow(\neg p(x))\}$} }
            child { 
            node[tableau] (c211){\underline{$\{\neg p(x)\}$} }
            node[below=1mm]{\xmark} 
            }
          }
        }
      child {
        node[tableau] (2) {$\{\underline{\tomorrow \psi}\}$}
     } edge from parent
    }}
  ;
  \end{tikzpicture}
  \caption{Example of application of the $\m{PRUNE}$ rule from \exaref{3}.}
  \label{fig:example3}
\end{figure}

\begin{example}
\label{exa:3}
Consider the following unsatisfiable formula interpreted over \EUF, for a unary
predicate $p$:
\begin{equation}
  \psi := \eventually (p(\nextvar x) \wedge
\tomorrow(\neg p(x)))
\end{equation}

The corresponding tableau is shown in \figref{example3}. Let $\vec u$ be the
rightmost branch with poised nodes $\vec \pi = \langle \pi_0, \pi_1\rangle$. 
We have $\Gamma(\pi_0) = \Gamma(\pi_1)$, and $\hist(\vec \pi_{\leq 0}) =
\hist(\vec \pi_{\leq 1}) = \top$. Thus the \textsf{PRUNE} rule applies, and
$\vec u$ can be rejected.
\end{example}

Since the $\m{PRUNE}$ rule can only reject (but not accept) branches, it may only affect
\emph{completeness}, but not soundness. 
As we prove in the remainder of this section, completeness of the
tableau calculus of \cite{GeattiGG22} is indeed preserved when augmented with
the $\m{PRUNE}$ rule. 

\paragraph{Completeness}

Here, we extend the completeness result of \cite{GeattiGG22,GeattiGG22arXiv} to
account for the additional $\m{PRUNE}$ rule. We start by defining a \emph{pre-model}, an abstract structure summarising the important aspects of
a state sequence in a tableau branch.

\begin{definition}[Atom]
  An \emph{atom} $\Delta$ for an \LTLfMT formula $\phi$ is a set $\Delta
  \subseteq \closure{(\phi)}$ such that: 
  \begin{compactenum} 
    \item the conjunction of all first-order formulas in $\Delta$ is
      $\TT$-satisfiable;
    \item for all $\psi \in \Delta$ to which a rule from
      \tabref{expansion:rules} applies, either $\Gamma_1 \subseteq \Delta$, or
      $\Gamma_2 \neq \emptyset$ and $\Gamma_2 \subseteq \Delta$; and 
    \item $\Delta$ is
      closed under logical deduction as far as $\closure{(\phi)}$ is concerned.
  \end{compactenum}  
\end{definition}

\begin{definition}
\label{def:pre:model}
A \emph{pre-model} for $\phi$ is a sequence of atoms $\vec \Delta = \langle\Delta_0, \dots, \Delta_{n-1}\rangle$ such that $\phi \in \Delta_0$, and for all $i$, $0\leq i < n$:
\begin{compactenum}
\item $\Delta_{n-1}$ does not contain any $p(t_1, \dots, t_k)$ where $\nextvars$ occurs,
\item if $\tomorrow \phi' \in \Delta_i$ then $i<n-1$ and $\phi'\in \Delta_{i+1}$,
\item if $\weaktomorrow \phi' \in \Delta_i$ then $i=n-1$ or $\phi'\in \Delta_{i+1}$,
\item if $\phi_1 \until \phi_2 \in \Delta_i$ then there is some $i\leq j < n$ such that
$\phi_2 \in \Delta_j$ and $\phi_1 \in \Delta_k$ for all $i\leq k < j$,
\item if $\phi_1 \release \phi_2 \in \Delta_i$ then either $\phi_2 \in \Delta_k$ for all $i\leq k < n$, or there is some $i\leq j < n$ such that
$\phi_1 \in \Delta_j$ and $\phi_2 \in \Delta_k$ for all $i\leq k \leq j$, and
\item all $\Delta_i$ are minimal with respect to set inclusion.
\end{compactenum}
\end{definition}

Let $F(\Delta)$ be the conjunction of all first-order formulas in an atom $\Delta$. Given a
pre-model $\vec \Delta = \langle \Delta_0, \dots, \Delta_{n-1} \rangle$, we say
that $\vec \Delta$ is \emph{satisfiable} if $\Omega(\langle
F(\Delta_0),\ldots,F(\Delta_{n-1})\rangle) \land \neg\ell$ is $\TT$-satisfiable.

Following~\cite{GeattiGG22,GeattiGG22arXiv}, one can show that from any pre-model
for an \LTLfMT formula $\phi$ one can obtain a model of $\phi$, and \viceversa, any model
of $\phi$ can be represented by a pre-model:

\begin{proposition}[\cite{GeattiGG22,GeattiGG22arXiv}]
\label{prop:sat:pre-model}
An \LTLfMT formula $\phi$ is satisfiable if and only if it has a satisfiable
pre-model.
\end{proposition}

There is a precise connection between pre-models of a formula and branches of
the tableau. In particular, the following extraction lemma can be proved, as in
\cite[Lem. 2 in Appendix A]{GeattiGG22arXiv}.

For a node $u$ in a tableau for $\phi$, let the \emph{atom of $u$}, denoted
$\Delta(u)$, be the set of all formulas in $\closure(\phi)$ that are entailed by
$\Gamma(u)$.

\begin{proposition}[\cite{GeattiGG22,GeattiGG22arXiv}]
\label{prop:extraction}
If $\vec\Delta\,{=}\,\langle\Delta_0, \dots, \Delta_{n{-}1}\rangle$ is a satisfiable pre-model for $\phi$, every complete tableau for $\phi$ has a branch with step nodes $\vec \pi = \langle \pi_0, \dots, \pi_{n{-}1}\rangle$ such that $\Delta(\pi_i)\,{=}\,\Delta_i$ for all $0\,{\leq}\,i\,{<}\,n$.
\end{proposition}

To prove completeness, we have to show that if a formula $\phi$ is satisfiable, there
is an accepted branch. As $\phi$ is satisfiable, it has a model, and by
\cref{prop:sat:pre-model}, there is also a satisfiable pre-model
$\vec\Delta=\langle\Delta_0, \dots, \Delta_{n-1}\rangle$ for $\phi$. Thus, by
\cref{prop:extraction}, there is a branch $\vec \pi = \langle \pi_0, \dots,
\pi_{n-1}\rangle$ in the tableau such that $\Delta(\pi_i) = \Delta_i$ for all $i$,
$0\,{\leq}\,i\,{<}\,n$. It is easy to see that (a prefix of) $\vec\pi$ cannot be
rejected by the $\m{CONTRADICTION}$ rule, as otherwise $\vec\Delta$ would
not be a satisfiable pre-model. However, it remains to show that $\vec\pi$
cannot be rejected by the $\m{PRUNE}$ rule. To this end, we first,
define a \emph{redundant segment} of a pre-model, \ie a segment that can be
safely removed from a satisfiable pre-model to obtain another, shorter,
satisfiable pre-model. Then, we show that if there are no redundant segments,
the tableau branch extracted by \cref{prop:extraction} cannot be
rejected by $\m{PRUNE}$.
To do so, we extend our notion of history constraints to pre-models
in a natural way, that is, given a pre-model $\vec\Delta=\langle\Delta_0,
\dots, \Delta_{n-1}\rangle$, we define $\hist(\vec\Delta)=\hist(\langle
F(\Delta_0),\ldots,F(\Delta_{n-1})\rangle)$.

\begin{definition}[Redundant segment]
  \label{def:redundant}
  Let $\Delta = \langle\Delta_0, \dots, \Delta_{n-1}\rangle$ be a pre-model for $\psi$ and $j<k< n$. Then the subsequence $\vec \Delta_{[j+1, k]}$ is \emph{redundant} if
  $\Delta_j = \Delta_k$ and $\hist(\vec \Delta_{\leq k}) \models_\TT \hist(\vec \Delta_{\leq j})$.
\end{definition}

Intuitively, a redundant segment can be removed from a pre-model because it does
no useful work towards the satisfaction of the formula. To show this, we need an
auxiliary result about history constraints. First, given two state variable assignments
$\alpha$ and $\alpha'$ we define the combination $\combine{\alpha}{\alpha'}$ of
them as a variable assignment with domain $V \cup \nextvars \cup \wnextvars$  by
setting $(\combine{\alpha}{\alpha'})(v) = \alpha(v)$ and
$(\combine{\alpha}{\alpha'})(\nextvar v) = (\combine{\alpha}{\alpha'})(\wnextvar
v) =\alpha'(v)$ for all $v\in V$. That is, $\alpha$ is used to interpret the
current state variables, and $\alpha'$ to interpret the variables at the next
state. Let $\vec C=\seq{C_0, \dots, C_{m-1}}$ be a sequence of first-order
formulas with free variables $V \cup \nextvars \cup \wnextvars$. Given a model
$M$, and a sequence of state variable assignments $\vec \alpha=\seq{\alpha_0,
\dots, \alpha_m}$, we write $M,\vec \alpha \models \vec C$ if $M,
\combine{\alpha_i}{\alpha_{i+1}} \models C_{i}$ for all $0 \leq i < m-1$, and
$M, \combine{\alpha_{m-1}}{\alpha_{m}} \models L(C_{m-1})$. We then have the
following relationship between satisfying assignments for history constraints,
and sequences of assignments that satisfy each constraint in the sequence
individually (similar as  ~\cite[Lemma 3.5]{FMW22}):

\begin{restatable}{lemma}{lemmaHistory}
\label{lem:history}
Let $M$ be a $\Sigma$-structure and $\vec C=\seq{C_0, \dots, C_{m-1}}$ be a sequence of first-order formulas with free variables $\VV \cup \nextvars \cup \wnextvars$, for $m\geq 1$.
\begin{compactenum}
\item[(1)] If $M, \seq{\alpha_0, \dots \alpha_m} \models \vec C$ then
$M, \alpha_m \models h(\vec C)$.
\item[(2)] If $M, \alpha \models h(\vec C)$ then there is a sequence
$\vec \alpha = \seq{\alpha_0, \dots \alpha_m}$ with $\alpha_m=\alpha$ such that $M,\vec \alpha \models \vec C$.
\end{compactenum}
\end{restatable}
\begin{proof}
  Both items are shown by a straightforward induction proof (see the 
  Appendix).
\end{proof}

Using \defref{redundant} and \lemref{history}, we can now show that a satisfiable pre-model remains satisfiable after
removing a redundant segment.

\begin{restatable}{lemma}{redundantRemove}
\label{lem:redundant:segments}
Let $\vec \Delta = \langle \Delta_1, \dots, \Delta_{n-1}\rangle$ be a satisfiable pre-model for $\psi$ with redundant segment $\vec \Delta_{[j+1, k]}$.
Then $\vec \Delta' = \vec \Delta_{\leq j}\vec \Delta_{>k}$ is a satisfiable pre-model as well.
\end{restatable}
\begin{proof}
  See the Appendix.
\end{proof}

It is finally possible to prove the main completeness result.

\begin{thm}[Soundness and completeness] 
  \label{thm:sound:complete}
  Given a \LTLfMT formula $\psi$, the tableau for $\psi$ augmented with the
  $\m{PRUNE}$ rule has an accepted branch if and only if $\psi$ is satisfiable.
\end{thm}
\begin{proof}
  As soundness is not affected by the $\m{PRUNE}$ rule, we are only concerned
  with completeness. Hence, suppose $\phi$ is satisfiable. By
  \cref{prop:sat:pre-model} there is a satisfiable pre-model $\vec
  \Delta=\langle \Delta_0\ldots,\Delta_{n-1}\rangle$ for $\phi$. Without loss of
  generality, we can assume that $\vec\Delta$ is of minimal length. By
  \cref{prop:extraction}, the tableau for $\phi$ has a corresponding branch
  $\vec u$ with poised nodes $\vec\pi=\langle \pi_0,\ldots,\pi_{n-1}\rangle$ such
  that $\Delta(\pi_i)=\Delta_i$ for all $0\le i < n$.
  As we mentioned, $\vec u$ cannot have been rejected by the $\m{CONTRADICTION}$
  rule. Now, suppose by contradiction that $\vec u$ has been rejected by the
  $\m{PRUNE}$ rule. Then, there is a node $\pi_i$ with
  $\Gamma(\pi_i)=\Gamma(\pi_n)$ and
  $\hist(\vec\pi)\models_{\TT}\hist(\vec\pi_{\le i})$. But then, we have that
  $\Delta_i=\Delta_n$ and $\hist(\vec\Delta)=\hist(\Delta_{\le i})$. That is,
  $\Delta_{[i,n]}$ is a redundant segment. By \lemref{redundant:segments}, we
  can remove it, obtaining a \emph{shorter} satisfiable pre-model $\vec\Delta_{<
  i}$. But this contradicts the assumption that $\vec\Delta$ was of minimal
  length. Hence, $\vec u$ cannot have been rejected by $\m{PRUNE}$,
  and is thus an accepted branch.
\end{proof}

\section{Decidable fragments}
\label{sec:decidable}

The new $\m{PRUNE}$ rule is not capable of pruning \emph{all} potentially infinite
branches in all possible case, since \LTLfMT is undecidable. However, we can
identify a general sufficient condition for this to happen, given that the underlying theory $\TT$ is decidable (which we assume throughout this section).

\begin{definition}[Finite memory]
  Given an \LTLfMT formula $\phi$, the \emph{history set} of $\phi$ is the set
  of all the formulas $\hist(\vec\Delta_{\le i})$ for any pre-model $\vec\Delta$
  of $\phi$ and any $0\le i< |\vec\Delta|$. A formula $\phi$ has \emph{finite
  memory} if its history set is \emph{finite up to $\TT$-equivalence}.
\end{definition}

\begin{thm}[Termination]
\label{thm:termination}
  The tableau for an \LTLfMT formula with \theproperty is finite.
\end{thm}
\begin{proof}
  As accepted or rejected branches are finite by definition, we are only concerned
  with branches that continue to expand forever without triggering any termination
  rule. Suppose $\phi$ has finite memory 
  but the tableau is infinite.
  Then there is at least one infinite branch since the
  branching degree is finite; let $\vec\pi=\langle\pi_0,\pi_1,\ldots\rangle$ be
  the poised nodes of this branch. For each prefix $\vec\pi_{\le i}$ for
  $i\ge0$, one can check that the sequence $\vec\Delta=\langle
  \Delta(\pi_0),\ldots,\Delta(\pi_i)\rangle$ is a pre-model for $\phi$. Since
  $\phi$ has finite memory, its history set is finite up to $\TT$-equivalence.
  As the possible labels of tableau nodes are also finite, for some $i$ large
  enough there exists a $j<i$ such that $\Gamma(\pi_j)=\Gamma(\pi_i)$ and
  $\hist(\vec\Delta_{\le j})\equiv_{\TT}\hist(\vec\Delta_{\le i})$, which means
  that $\hist(\vec\pi_{\le j})\models_{\TT}\hist(\vec\pi_{\le i})$. Hence the
  $\m{PRUNE}$ rule would apply to $\vec\pi_{\le i}$, contradicting the
  hypothesis that no termination rule is triggering along $\vec\pi$.
\end{proof}

While \thmref{termination} gives only a semantic and, in general, undecidable
condition for termination, we now show several concrete, effectively
identifiable classes of \LTLfMT formulas having finite memory.
Indeed, we use this approach to both re-prove and extend decidability
conditions previously obtained by ad-hoc methods in the literature, and to show
novel results conditions for other relevant classes of formulas. 

Before giving details, we summarise our decidability results.
To this end, let the set of \emph{iteration conditions} 
of an \LTLfMT formula $\phi$ consist of all
literals that occur in $\phi_1$ for any subformula $\phi_1 \until \phi_2$ of $\phi$, or in $\psi_2$ for any subformula $\psi_1 \release \psi_2$ of $\phi$.
We show decidability for the following classes of \LTLfMT formulas:
\smallskip
\begin{compactitem}
\item[(\textsf{NCS})]
Formulae without \emph{cross-state comparisons}, \ie that have no occurrences of $\nextvars{\cup}\wnextvars$, \eg 
$(x{>}y \until x{+}y\,{=}\, 2z) \wedge \always (x{+}y{>}0)$;
\item[(\textsf{FX})]
Formulas where the only temporal operators are $\eventually$, $\tomorrow$, and $\weaktomorrow$, \eg
$\eventually (p(\nextvar x) \wedge \tomorrow(\neg p(x))) \wedge \tomorrow \eventually (r(x,y) \vee r(\nextvar x, y))$;
\item[(\textsf{BL})]
\emph{Bounded lookback} formulas, that generalize the above two by requiring that constraint interaction via $\nextvars$ and $\wnextvars$ is restricted to finitely many configurations, \eg $p(x, \nextvar y) \until (\nextvar x =  x + y)$.
\item[(\textsf{MC})]
Formulas over \LRA where all iteration conditions are \emph{monotonicity constraints}, \ie variable-to-variable or variable-to-constant comparisons. An example is the formula in \exaref{1}.
\item[(\textsf{IPC})]
Formulas over \LIA where all iteration conditions are \emph{integer periodicity constraints}, \eg $(y \equiv_3 x) \until (x > 42) \wedge \eventually (x{+}y=z)$.
\end{compactitem}
\smallskip

Demri and d'Souza~\cite{DD07,Demri06} showed that satisfiability is decidable for \LTLfMT over arithmetics where \emph{all} literals are monotonicity or integer periodicity constraints, but our results \textsf{(MC)} and \textsf{(IPC)} show that is suffices to restrict the shape of iteration conditions respectively.
To the best of our knowledge, the result \textsf{(FX)} is novel; and \textsf{(BL)} is novel as a decidability result for satisfiability, though a similar result is known for model checking over \LTLf with arithmetic~\cite{FMW22}, and for the more restrictive condition of \emph{feedback freedom} also supporting the theory \EUF~\cite{DDV12}.
In the remainder of this section, we formally prove decidability for the five classes above.

We start with \emph{bounded lookback} formulas. To formally define this class of
formulas, we use the structure of a \emph{dependency graph} to capture the
dependencies between variables induced by a pre-model.
\begin{definition}[Dependency graph]
  \label{def:depenency:graph}
  Let $\vec \Delta = \langle \Delta_0, \dots, \Delta_{n-1} \rangle$ be a
  pre-model. Its \emph{dependency graph} is $\DG(\vec \Delta) = (\VV^{\leq n},
  E^=, E^{\neq})$ where $\VV^{\leq n} = V^0 \cup \dots \cup V^{n}$ is the set of
  nodes, and $E^=$ and $E^{\neq}$ are sets of two kinds of edges defined as follows.

  Two variables $x,y\in \VV^{\leq n}$ are \emph{dependent} if there is a
  sequence of variables $z_0,z_1, \dots, z_m\in\WW$ such that
  $\Omega(\vec\Delta)$ contains a literal $\ell_0$ mentioning $x$ and $z_0$, a
  literal $\ell_m$ mentioning $z_m$ and $y$, and, a literal $\ell_i$ that
  mentions both $z_i, z_{i+1}$ for all $1 \leq i < m$. In this case:
  \begin{compactitem}
  \item
    $(x,y)\in E^=$ if all the literals $\ell_i$ are equalities;
  \item
    $(x,y)\in E^{\neq}$ if at least one $\ell_i$ is not an equality.
  \end{compactitem}
\end{definition}

In other words, $E^=$ is the smallest equivalence relation on $\VV^{\leq n}$
that contains the transitive closure of all equality literals in
$\Omega(\vec\Delta)$, while $E^{\neq}$ captures connections by arbitrary other
kinds of literals. Moreover, let $\DG_=(\vec \Delta) = (\VV^{<n}, E^{\neq})$ be
the graph obtained from $\DG(\vec \Delta)$ by collapsing all equality edges to
an arbitrary element in the equivalence relation induced by $E^=$.

\begin{definition}
For $k \geq 0$, an \LTLfMT formula $\psi$ has \emph{$k$-bounded lookback} if for all pre-models $\vec \Delta$ of $\psi$, it holds that all acyclic paths in $\DG_=(\vec \Delta)$ have length at most $k$.
\end{definition}

A formula has \emph{bounded lookback} (BL) if it has $k$-bounded lookback for some $k$. The notion is an adaptation of a similar property used in model checking~\cite{FMW22}; and as shown there, it generalizes the notion of \emph{feedback freedom}~\cite{DDV12} developed to verify database systems.
Intuitively, bounded lookback expresses that in order to check whether a run satisfies $\phi$, it suffices to remember a bounded amount of information from past states. The next examples illustrate the idea.

\begin{figure}[t]
  \resizebox{\linewidth}{!}{
    \begin{tikzpicture}[xscale=.6, yscale=1.2]
      \node[scale=.7] at (-1,.5) {$x$};
      \node[scale=.7] at (-1,.1) {$y$};
      \foreach \i in {0,1,2,3,4} {
        \node[scale=.65] at (\i,1) {\i};
        \node[fill, circle, inner sep=0pt, minimum width=1mm] (x\i) at (\i,.5) {};
        \node[fill, circle, inner sep=0pt, minimum width=1mm] (y\i) at (\i,.1) {};
        }
      \draw (x0) -- (y1);
      \draw (x1) -- (y2);
      \draw (x2) -- (y3);
      \draw (x3) -- (x4) -- (y3) -- (x3);

      \begin{scope}[xshift=5.5cm]
        
        \foreach \i in {0,1,2,3} {
          \node[scale=.65] at (\i,1) {\i};
          \node[fill, circle, inner sep=0pt, minimum width=1mm] (x\i) at (\i,.5) {};
          \node[fill, circle, inner sep=0pt, minimum width=1mm] (y\i) at (\i,.1) {};
          }
        \draw (y0) -- (y1) -- (y2) -- (y3);
        \draw (x0) -- (x1) -- (x2) -- (x3);
        \draw[dotted, line width=0.8pt] (x2) -- (y2);
        \begin{scope}[xshift=4.5cm]
        
        \foreach \i in {0,1,2,3} {
          \node[scale=.65] at (\i,1) {\i};
          \node[fill, circle, inner sep=0pt, minimum width=1mm] (x\i) at (\i,.5) {};
          \node[fill, circle, inner sep=0pt, minimum width=1mm] (y\i) at (\i,.1) {};
          }
        \draw (y0) -- (y1) -- (x2) -- (y3);
        \draw (x0) -- (x1) -- (x2) -- (x3);
        \end{scope}
      \end{scope}
    \end{tikzpicture}
  }
  \caption{Dependency graphs for the formulas in \exaref{dependency1} (left) 
  and \exaref{dependency2} (center and right). Equality edges are drawn dotted 
  and other edges solid.}
  \label{fig:dependency}
\end{figure}
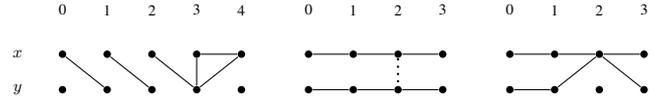

\begin{example}
  \label{exa:dependency1}
For $\phi = p(x, \nextvar y) \until (\nextvar x =  x + y)$
consider the pre-model
$\vec \Delta = \langle \Delta_0, \Delta_0, \Delta_0, \Delta_1\rangle$
where
$\Delta_0 = \{p(x, \nextvar y), \tomorrow\phi\}$ and
$\Delta_1 = \{\nextvar x =  x + y\}$. We have:
\begin{equation}
\Omega(\vec\Delta) = p(x_0, y_1) \land p(x_1, y_2) \land p(x_2, y_3)
\land x_4 =  x_3 + y_3
\end{equation}
Then, $\DG(\vec\Delta)$ is pictured in \figref{dependency} (left), representing all the connections between the variables $x_0,y_0, \dots, x_4,y_4$ implied by $\Omega(\vec \Delta)$.

Since there are no equality literals, $\DG_=(\vec\Delta)$ coincides with $\DG(\vec\Delta)$. The longest acyclic path in $\DG_=(\vec\Delta)$ has length 3.
Though $\phi$ has infinitely many pre-models, it can be seen that in all their DGs, acyclic paths have length $\leq 3$, so
$\phi$ has 3-bounded lookback.
\end{example}

\begin{example}
\label{exa:dependency2}
For the pre-model $\vec \Delta = \langle \Delta_0, \Delta_1,
\Delta_2\rangle$ for $\psi$ from \exaref{1}, where 
$\Delta_0 = \{\psi, \psi', x\,{<}\, 0,
y\,{=}\,1, \nextvar y\,{>}\,y, \nextvar x\,{\leq}\,x, \tomorrow\psi'\}$,
$\Delta_1 = \{\nextvar y\,{>}\,y, \nextvar x\,{\leq}\,x, \psi',\tomorrow\psi'\}$, and
$\Delta_2 = \Delta_1 \cup\{x=y\}$, we have
\begin{align}
\Omega(\vec\Delta) = {} & x_0\,{<}\,0 \land y_0\,{=}\,1 \land
y_1 \,{>}\, y_0 \land x_1\,{\leq}\, x_0 \land 
y_2\,{>}\, y_1  \\
{} \land {} & x_2\,{\leq}\, x_1 \land 
y_3\,{>}\,y_2 \land x_3\,{\leq}\,x_2 \land x_2\,{=}\,y_2
\end{align}
\figref{dependency} shows $\DG(\vec\Delta)$ (center) and $\DG_=(\vec\Delta)$ (right). The longest path in $\DG_=(\vec\Delta)$ has length 4. However, $\psi$ has infinitely many pre-models
$\vec\Delta_m = \langle \Delta_0, \Delta_1, \dots,\Delta_1,\Delta_2\rangle$
with $m$ repetitions of $\Delta_1$, for any $m\geq 0$, which have similar DG$_=$'s with paths of length $2(m+1)$. So $\psi$ does not have $k$-bounded lookback, for any $k$.
\end{example}

The proof of the following result recasts the approach from \cite[Thm. 5.10]{FMW22} for pre-models and satisfiability.

\begin{thm}
\label{thm:bounded:lookback}
Satisfiability of BL formulas is decidable.
\end{thm}
\begin{proof}
Let $\phi$ have $k$-bounded lookback, and $\vec \Delta$ a pre-model of length $n$ for $\phi$.
The history constraint $h(\vec \Delta)$ encodes
$\DG(\vec \Delta)$. Let $\chi$ be the formula obtained from  $h(\vec \Delta)$ by removing all equalities between variables and replacing each variable in $\VV^{\leq n}$ by
a representative from its $E^=$-equivalence class.
Then $\chi \equiv_\TT h(\vec \Delta)$ and $\chi$ encodes $\DG_=(\vec \Delta)$.
Since all acyclic paths in $\DG_=(\vec \Delta)$ have length at most $k$,
each variable in $V^n$ is connected in $\DG_=(\vec \Delta)$ to at most $k$ variables in $V^{\leq n}$.
As $\chi$ encodes $\DG_=(\vec \Delta)$, $\chi$ is equivalent
to a formula with at most $k\cdot |\VV|$ quantified variables.
All literals in $\chi$ are (renamed) first-order formulas in $\phi$.
The number of formulas with a bounded number of quantifiers and finite vocabulary is finite up to equivalence, so $\psi$ has \theproperty, and by \thmref{termination}, the tableau is finite.
\end{proof}

Note that for a given $k$ and \LTLfMT formula $\psi$, it is decidable whether $\psi$ has $k$-bounded lookback, by checking whether none of the finitely many (prefixes of) pre-models of length $k+1$ has a path in DG$_=$ of length more than $k$ (\cf \cite{FMW22}).
However, it is undecidable whether there is \emph{some} $k$ such that $\psi$ has $k$-bounded lookback.

Let a formula have \emph{cross-state comparisons} if it contains variables in
$\nextvars$ or $\wnextvars$. Note that for formulas without cross-state
comparisons, dependency graphs have only edges from some $x_i$ to some $y_i$ for
the same $i$ (\ie vertical edges if pictured as in \figref{dependency}), so
all acyclic paths have length at most $|\VV|$. We hence obtain the following:

\begin{corollary}
Satisfiability of formulas without cross-state comparisons is decidable.
\end{corollary}

Now, let an \LTLfMT formula be an \emph{$\m{FX}$ formula} if its only temporal operators are $\eventually$, $\tomorrow$, and $\weaktomorrow$.

\begin{thm}
\label{thm:FX}
Satisfiability of $\m{FX}$ formulas is decidable.
\end{thm}
\begin{proof}
Suppose an $\m{FX}$ formula $\phi$ contains $m$ literals, and let
$\vec \Delta = \langle \Delta_0, \dots, \Delta_{n-1}\rangle$ be a pre-model for it.
By the expansion rules of the $\eventually$, $\tomorrow$, and $\weaktomorrow$ operators, and the minimality of atoms, every literal occurrence in $\phi$ corresponds to at most one occurrence in the pre-model.
Thus, $\vec \Delta$ contains at most $m$ literals overall, and each path in its
dependency graph is upper-bounded by $m\cdot |\VV|$, hence $\phi$ has bounded
lookback. The claim then follows from \thmref{bounded:lookback}.
\end{proof}

We next consider fragments of \LTLfMT over arithmetic theories.
\emph{Monotonicity constraints} (MC) restrict linear arithmetics over the
rationals, demanding all constraints to be of the form $p \odot q$ where $p,q\in
{\mathbb Q\,{\cup}\,\VV\,{\cup}\,\nextvars\,{\cup}\,\wnextvars}$ and $\odot\in\{=, \neq, \leq, <\}$. An \LTLfMT formula $\phi$ is an \emph{MC
formula} if all literals in $\phi$ are MCs, such as in the formula from \exaref{1}.
Satisfiability of MC formulas is known to be decidable \cite[Cor. 5.5]{DD07}.
Here, we prove decidability for a larger class.
\begin{definition}[Quasi-MC formulas]
  \label{def:quasi-MC}
  An \LTLfMT formula over the signature of \LRA is \emph{quasi-MC} if all its \emph{iteration conditions} are MCs.
\end{definition}

E.g., $(\nextvar x\,{>}\,x \wedge \nextvar y\,{>}\,y) \until (x{+}y\,{>}\,10)$
is not an MC-, but a quasi-MC formula.
MC formulas are important in BPM, as they can model decision tables~\cite{deLeoniFM21Jods}.
To show decidability of quasi-MC formulas, we use the following fact about quantifier elimination~\cite[Sec. 5.4]{KS16}:
if $\phi$ is an \LRA formula where all literals are MCs over a set of constants $\mc K$ and variables $X \cup \{x\}$,
then one can compute a formula $\phi' \equiv_{\m{LRA}} \exists x.\, \phi$ such that all literals in $\phi'$ are MCs over constants $\mc K$ and variables $X$; \eg using a Fourier-Motzkin procedure. 

\begin{thm}
\label{thm:MCs}
Satisfiability of quasi-MC formulas is decidable.
\end{thm}
\begin{proof}
Let $\mc K$ be the set of constants, $I$ the set of iteration conditions,
$A$ the set of all first-order formulas in a quasi-MC formula
$\phi$, and $m$ the number of occurrences of formulas of $A$ in $\phi$. For a
pre-model $\vec \Delta = \langle \Delta_0, \dots, \Delta_{n-1}\rangle$, let 
$J = \{i_1, \dots, i_k\} \subseteq \{0, \dots, n{-}1\}$ be all indices
such that $F(\Delta_{i_j})$ contains a formula in $A \setminus I$.
\emph{W.l.o.g.}, assume that $n-1 \in J$; otherwise the reasoning is
similar. Note that $k \leq m$ since every occurrence of a first-order formula
in $\phi$ that is not an iteration condition can occur in at most one atom in a
pre-model. Now, $\Omega(\vec \Delta)$ has free variables $\VV^{\leq n} =
V^0 \cup \dots V^n$; let $X \subseteq \VV^{\leq n}$ be the set of variables occurring in
$\{F(\Delta_{j})^{(j)} \mid j\in J\}$, and $Y = \VV^{\leq n} \setminus
X$. Then we can write $h(\vec \Delta)$ as 
\begin{equation}
\biggl(\exists X. \bigl (\exists Y. \bigwedge_{i\in N\setminus J} C_i^{(i)} \bigr) \wedge \bigwedge_{i\in J\setminus \{m-1\}} C_i^{(i)} \wedge L(C_{m-1})^{(m-1)}\biggr)[\vec \VV]
\end{equation}
where $C_i = F(\Delta_i)$.
By the QE property of MCs,
the subformula $\exists Y. \bigwedge_{i\in N\setminus J} C_i^{(i)}$
is \LRA-equivalent to a first-order formula $\chi$ where all literals are MCs over constants $\mathcal K$ and variables $\VV \cup X$. There are only finitely many such $\chi$ up to equivalence, as there are only finitely many MCs over a finite set of variables and constants.
Moreover, the number of possibilities for the sequence $C_{i_1}, \dots, C_{i_k}$ is bounded by $2^{2^m}$ since all these $C_{i_j}$ must be conjunctions of subsets of $A\,{\setminus}\,I$, and $k\,{\leq}\,m$.
Thus, up to equivalence, there are  finitely many possibilities for $h(\vec \Delta)$, so the history set is finite.
\end{proof}

\emph{Integer periodicity constraints} (IPCs) restrict linear \emph{integer} arithmetic (\LIA) and are \eg used in calendar formalisms~\cite{Demri06}.
Precisely, IPC atoms have the form $x = y$ or $x \odot d$ for $\odot \in \{=,\neq, <, >,\equiv_k\}$, or $x \equiv_k y + d$, for variables $x,y$ with domain $\mathbb Z$ and $k,d\in \mathbb N$. 
An \LTLfMT formula $\phi$ over \LIA is an \emph{IPC formula} if all first-order formulas in $\phi$ are IPCs, and a \emph{quasi-IPC formula} if all iteration conditions are IPCs.
IPC formulas are known to be decidable~\cite[Thm. 3]{Demri06}.

We extend this result to quasi-IPC formulas by using a quantifier elimination property as for MCs:
if $\phi$ is a first-order formula where all literals are IPCs over a set of constants $\mc K$ and variables $X \cup \{x\}$,
then one can compute a formula $\phi' \equiv_\LIA \exists x.\, \phi$ such that $\phi'$ is a first-order formula where all literals are IPCs over constants $\mc K$ and variables $X$ \cite[Thm. 2]{Demri06}. 
Then, the following can be proven exactly like \thmref{MCs}, using the fact that there are only finitely many \LIA formulas where all literals are IPCs over finite sets of variables and constants:

\begin{thm}
Satisfiability of quasi-IPC formulas is decidable.
\end{thm}




\section{Related work and conclusions}
\label{sec:conclusions}

In this paper we considered the satisfiability problem for \LTLfMT, a highly expressive extension of \LTLf. In earlier work, a tableau system for \LTLfMT was proposed that is, however, incomplete to show unsatisfiability. In this paper, we proposed a pruning rule for this tableau that we proved sound and complete. We show that the tableau construction terminates whenever the \LTLfMT formula satisfies the semantic property of \theproperty, and use this abstract termination condition to prove decidability for several concrete, checkable, and relevant classes of formulas, extending results from the literature.

\smallskip

Given the limited expressivity of propositional \LTL, several extensions with richer background theories have been considered, in particular (fragments of) arithmetic theories~\cite{DD07, Demri06,DDV12,DLV19}.
The extension of \LTL with first-order theories is highly challenging, as even the most basic verification tasks become undecidable~\cite{CDMP22}.
A starting point for this work is the \LTLfMT tableau by Geatti \etal~\cite{GeattiGG22}, which provides a semi-decision procedure; but, lacking a pruning rule, is rarely able to show unsatisfiability, and no decidability results for fragments of \LTLfMT are given.
However, some decidability results for model checking and satisfiability (which are equivalent in linear-time temporal logics) for \LTL with more specific theories are known.
Demri and D'Souza~\cite{DD07} showed that satisfiability of LTL with monotonicity constraints (MCs), over both integers and rationals, is decidable in PSPACE, and the same holds for LTL over integer periodicity constraints~\cite{Demri06}.
Our results for the (\textsf{MC}) and (\textsf{IPC}) fragments strictly extend these decidability results, since we only restrict iteration conditions of formulas.
The picture gets more diverse for branching-time temporal logics equipped with similar arithmetic theories; in this case, satisfiability and model checking do no longer coincide~\cite{Cerans94,CKL16,Gascon09,FMW22c}.
Damaggio \etal~\cite{DDV12} considered LTL model checking for transition systems that operate over databases and include arithmetic conditions, and proved decidability if the system together with the \LTL formula satisfies the property of \emph{feedback freedom}.
For purely arithmetic transition systems, feedback freedom was extended by Felli \etal to that of \emph{bounded lookback}~\cite{FMW22}.
Our decidability result for (\textsf{BL}) takes this idea to arbitrary theories, and recasts it for the satisfiability problem, thus strictly extending~\cite{DDV12,FMW22}. We showed that in the context of satisfiability, (\textsf{BL}) implies decidability of the (\textsf{FX}) fragment, which has no counterpart in model checking.
Deutsch \etal~\cite{DLV19} proved decidability of model checking for hierarchic transition systems and a restricted variant of \LTL (HLTL-FO), but this logic is in general incomparable to \LTLfMT.
Our notion of \emph{history constraints} is inspired by the respective notions from~\cite{FMW22,DDV12}, though we recast it here for satisfiability and in the setting of a tableau system.\fitpar

Tableau systems for \LTL and extensions thereof have been extensively
considered~\cite{LichtensteinP00,Schwendimann98,Reynolds16a,GeattiGMR20}. The
tableau for \LTLfMT provided in \cite{GeattiGG22} is based on Reynolds' one-pass
and tree-shaped tableau for \LTL~\cite{Reynolds16a}, whose $\m{PRUNE}$ rule does
not transfer directly to the first-order case. Tableau calculi for first-order
extensions of \LTL have also been proposed~\cite{KontchakovLWZ04}, but they are
not parameterised over the underlying theory, and the considered logic do not
support $\nextvar$ and $\wnextvar$ terms.

\smallskip

Several directions for future work can be considered. Following the path taken
by \cite{GeattiGG22}, an SMT encoding of our $\m{PRUNE}$ rule would allow for its implementation in the BLACK temporal reasoning framework~\cite{GeattiGM19}.
Moreover, whether these results can be extended to a
version of \LTLfMT supporting time-varying relations is still open.
Finally, we want to study also other, related tasks such as 
branching-time logics modulo theories, and \LTLfMT monitoring~\cite{FMPW23}. 

\ack This work was partially funded by the UNIBZ project ADAPTERS, and the PRIN MIUR project PINPOINT Prot. 2020FNEB27. Nicola Gigante acknowledges the support of the PURPLE project, 1st Open 
Call for Innovators of the AIPlan4EU H2020 project, a project funded by EU Horizon 2020 research and innovation programme under GA n. 101016442 (since 2021).

\bibliography{references}

\appendix
\clearpage
\section{Proofs}

\lemmaHistory*
\begin{proof}
Both items are by induction on $m$.

(1)
If $m=1$ and $M, \langle\alpha_0, \alpha_1\rangle \models \langle C_0\rangle$ then $M, \combine{\alpha_{0}}{\alpha_{1}} \models L(C_{0})$, so after renaming and quantification, 
\[M, \alpha_{1} \models (\exists V^0.\: L(C_0)^{(0)})[\vec \VV] = h(\langle C_0\rangle).\]

For the induction step, suppose $\vec C=\seq{C_0, \dots, C_{m}}$ and
$M, \seq{\alpha_0, \dots \alpha_{m+1}} \models \vec C$.
Let $M'$ be like $M$ but such that $M' \models \ell$.
For $\vec C'=\seq{C_0, \dots, C_{m-1}}$, we have 
$M', \seq{\alpha_0, \dots \alpha_{m}} \models \vec C'$.
By the induction hypothesis, $M', \alpha_{m} \models h(\vec C')$.
Since $M' \models \ell$, it also holds that $M', \alpha_{m} \models (\exists V^0 \dots V^{m-1}.\bigwedge_{i=0}^{m-1} C_i^{(i)})[\vec \VV]$, \ie, $M'$ and $\alpha_m$ satisfy the formula that is like  $h(\vec C')$ but where $L$ is not applied to $C_{m-1}$; call this fact ($\star$).
Let $\alpha_{m}'$ be the substitution with domain $V^m$ such that $\alpha_{m}'(v^m) = \alpha_{m}(v)$ and 
$\alpha_{m+1}'$ have domain $V_{m+1}$ such that $\alpha_{m+1}'(v^{m+1}) = \alpha_{m+1}(v)$
for all $v\in \VV$, so they are like $\alpha_{m}$ and $\alpha_{m+1}$, respectively, but with domains
$V^m$ and $V^{m+1}$.
Since $M, \seq{\alpha_0, \dots \alpha_{m+1}} \models \vec C$, we have $M, \combine{\alpha_m}{\alpha_{m+1}} \models L(C_m)$, so $M, \alpha_m' \cup \alpha_{m+1}' \models L(C_m)^{(m)}$.
From ($\star$) we have
$M, \alpha_{m}' \models \exists V^0 \dots V^{m-1}.\bigwedge_{i=0}^{m-1} C_i^{(i)}$ (using $M$ instead of $M'$, as $\ell$ is not involved).
By combining this with the above, we have 
$M, \alpha_m' \cup \alpha_{m+1}'  \models \exists V^0 \dots V^{m-1}.\bigwedge_{i=0}^{m-1} C_i^{(i)} \wedge L(C_m)^{m}$, so 
$M, \alpha_{m+1}'  \models \exists V^0 \dots V^{m}.\bigwedge_{i=0}^{m-1} C_i^{(i)} \wedge L(C_m)^{m}$, 
hence by renaming variables, $M, \alpha_m \models h(\vec C)$.

(2) Let $m=1$ and $M, \alpha \models h(\langle C_0\rangle)$, which means
$M, \alpha \models (\exists V^0.\: L(C_0))^{(0)}[\vec \VV]$.
Let $\alpha_1'$ have domain $V^1$ such that $\alpha_1'(v^1) = \alpha(v)$ for all $v\in \VV$.
There must be an assignment $\alpha_0'$ with domain $V^0$ such that
$M, \alpha_0' \cup \alpha_1' \models L(C_0)^{(0)}$, so for $\alpha_0$ with domain $V$ such that
$\alpha_0'(v^0) = \alpha_0(v)$ for all $v\in V$, it holds that
$M, \combine{\alpha_{0}}{\alpha} \models L(C_{0})$, so
$M, \langle\alpha_0, \alpha\rangle \models \langle C_0\rangle$.

For the induction step, let $\vec C=\seq{C_0, \dots, C_{m}}$,
$\vec C'=\seq{C_0, \dots, C_{m-1}}$, and suppose $M, \alpha \models h(\vec C)$, so
\[M, \alpha \models (\exists V^0 \dots V^{m}.\:\bigwedge_{i=0}^{m-1} C_i^{(i)} \wedge L(C_{m})^{m})[\vec \VV]\]
Let $\widehat \alpha$ have domain $V^{m+1}$ such that $\widehat \alpha(v^{m+1}) = \alpha(v)$
for all $v\in \VV$, so
$M, \widehat \alpha \models \exists V^0 \dots V^{m}.\:\bigwedge_{i=0}^{m-1} C_i^{(i)} \wedge L(C_{m})^{m}$.
Thus there is an assignment $\widehat \alpha'$ with domain $V^{m}$ such that
$M, \widehat\alpha \cup \widehat\alpha' \models \exists V^0 \dots V^{m-1}.\:\bigwedge_{i=0}^{m-1} C_i^{(i)} \wedge L(C_{m})^{m}$ ($\star$).
For $\alpha'$ with domain $\VV$ such that $\widehat\alpha'(v^{m}) = \alpha'(v)$
for all $v\in \VV$, it thus holds that
$M, \alpha' \models \exists V^0 \dots V^{m-1}.\:\bigwedge_{i=0}^{m-1} C_i^{(i)}$.
Let $M'$ be like $M$ but such that $M' \models \ell$. We have
$M', \alpha' \models (\exists V^0 \dots V^{m-1}.\:\bigwedge_{i=0}^{m-2} C_i^{(i)} \wedge L(C_{m-1})^{m-1})[\vec \VV] = h(\vec C')$.
By the induction hypothesis, there is a sequence $\seq{\alpha_0, \dots \alpha_{m}}$ such that
$M', \seq{\alpha_0, \dots \alpha_{m}} \models \vec C'$ and $\alpha_m = \alpha'$.
Since $M' \models \ell$, by definition of $L$, it holds that 
$M, \combine{\alpha_i}{\alpha_{i+1}} \models C_{i}$ for all $0 \leq i < m$ (where $C_{m-1}$ is not modified by $L$).
From ($\star$), we also have $M, \combine{\alpha'}{\alpha} \models L(C_{m})$, so
for $\vec \alpha = \seq{\alpha_0, \dots \alpha_{m}, \alpha}$ we have 
$M, \vec \alpha \models \vec C$.
\end{proof}

\redundantRemove*

\begin{proof}
First, we show that, $\vec {\Delta'}$ is still a pre-model for $\phi$:
Since $\Delta_j = \Delta_k$, for every $\tomorrow \psi \in \Delta_j$ it must hold that
$\psi \in \Delta_{k+1}$; and for every $\weaktomorrow \psi \in \Delta_j$, there is nothing to show if $k=n$, or otherwise $\psi \in \Delta_{k+1}$ must hold as well.
If $\psi_1 \until \psi_2 \in \Delta_j$ then $\psi_1 \until \psi_2 \in \Delta_k$, so the eventuality must be fulfilled at a later point, and similarly for $\release$.
Minimality with respect to set inclusion is clear.

It remains to show that $\vec \Delta'$ is satisfiable.
We abbreviate the first-order formulas in $\Delta_i$ by $\phi_i := \bigwedge F(\Delta_i)$ for all $0 \leq i < n$.
By assumption, $\vec \Delta$ is satisfiable, so
$\Omega(\langle \phi_1,\ldots,\phi{n-1}\rangle) \land \neg\ell$ is $\TT$-satisfiable.
Thus also
$\hist(\vec \Delta) \land \neg \ell$ is $\TT$-satisfiable, so there are a $\Sigma$-structure $M$ and a state variable assignment $\alpha$
such that $M, \alpha \models \hist(\vec \Delta) \land \neg \ell$ ($\star$).
By \lemref{history} there is a sequence
$\vec{\alpha} = \seq{\alpha_0, \dots, \alpha_n}$ 
such that $\alpha_n = \alpha$ and
$M,\vec{\alpha} \models \seq{\phi_1, \dots, \phi_n}$.
Let $M'$ be like $M$ except that $M' \models \ell$.
Then $M',\seq{\alpha_0, \dots, \alpha_k} \models \seq{\phi_1, \dots, \phi_k} \wedge \ell$.
By \lemref{history} it thus holds that $M', \alpha_k \models \hist(\vec \Delta_{\leq k})$.
Since $\hist(\vec \Delta_{\leq k}) \models_\TT \hist(\vec \Delta_{\leq j})$,
it holds that
$M', \alpha_k \models \hist(\vec \Delta_{\leq j})$.

Again by \lemref{history} there is a sequence
$\vec{\alpha'} = \seq{\alpha_0', \dots, \alpha_j'}$
such that $\alpha_j' = \alpha_k$ and
$M',\vec{\alpha'} \models \seq{\phi_0, \dots, \phi_j}$.
Since $M' \models \ell$, we have 
$M', \combine{\alpha_i}{\alpha_{i+1}} \models \phi_{i}$ for all $0 \leq i < j$.
With $\alpha_j' = \alpha_k$, it follows that the combined sequence 
$\vec{\alpha}'' = \seq{\alpha_0', \dots, \alpha_{j-1}', \alpha_{k}, \dots, \alpha_n}$
satisfies
$M', \vec{\alpha}'' \models \seq{\phi_1, \dots, \phi_{j-1}, \phi_k, \dots, \phi_n}$.
Again by \lemref{history}, $\hist(\vec \Delta')$ is $\TT$-satisfiable.
Finally, $M, \alpha_n \models \hist(\vec{\Delta'}) \land \neg \ell$ must hold
because  $M, \combine{\alpha_{n-1}}{\alpha_n} \models \phi_n \land \neg \ell$ follows from ($\star$). 
\end{proof}
\end{document}